\documentclass[11pt]{article}
\usepackage[margin=1in]{geometry}
\usepackage{amssymb,amsthm}
\usepackage{hyperref}       
\usepackage{url}            
\usepackage{booktabs}       
\usepackage{amsfonts}       
\usepackage{nicefrac}       
\usepackage{microtype}      
\usepackage{graphics}
\usepackage{caption}
\usepackage{subcaption}

\usepackage{amssymb,amsthm}
\usepackage{amsfonts}
\usepackage{latexsym,mathtools,enumitem,xcolor,ifthen}

\newcommand{\Reals}{\mathbb R}

\newcommand{\Prob}{\mathbb{P}}
\newcommand{\E}{\mathbb{E}}
\newcommand{\Ind}{\mathbb{I}}
\newcommand{\Var}{\text{Var}}
\newcommand{\Cov}{\text{Cov}}

\newcommand{\cE}{\mathcal{E}}
\newcommand{\cF}{\mathcal{F}}

\newcommand{\cN}{\mathcal{N}}

\newcommand{\cR}{\mathcal{R}}

\newcommand{\MSE}{\text{\normalfont MSE}}

\newboolean{showcomments}
\setboolean{showcomments}{true}
\newcommand{\todo}[1]{  \ifthenelse{\boolean{showcomments}}
	{ \textcolor{red}{(todo: #1)}} {}  }

\usepackage{booktabs} 
\usepackage[ruled]{algorithm2e} 

\SetAlFnt{\small}
\SetAlCapFnt{\small}
\SetAlCapNameFnt{\small}
\SetAlCapHSkip{0pt}
\IncMargin{-\parindent}

\allowdisplaybreaks

\newtheorem{theorem}{Theorem}
\numberwithin{theorem}{section}

\newtheorem{lemma}[theorem]{Lemma}

\newtheorem{assumption}{Assumption}

\begin{document}

\title{Nonparametric Matrix Estimation with One-Sided Covariates}
\author{Christina Lee Yu \\
		School of Operations Research and Information Engineering \\
		Cornell University}
\date{}
\maketitle

\begin{abstract}
Consider the task of matrix estimation in which a dataset $X \in \Reals^{n\times m}$ is observed with sparsity $p$, and we would like to estimate $\E[X]$, where $\E[X_{ui}] = f(\alpha_u, \beta_i)$ for some Holder smooth function $f$. We consider the setting where the row covariates $\alpha$ are unobserved yet the column covariates $\beta$ are observed. We provide an algorithm and accompanying analysis which shows that our algorithm improves upon naively estimating each row separately when the number of rows is not too small. Furthermore when the matrix is moderately proportioned, our algorithm achieves the minimax optimal nonparametric rate of an oracle algorithm that knows the row covariates. In simulated experiments we show our algorithm outperforms other baselines in low data regimes.
\end{abstract}

\section{Introduction}

Consider an user-product ratings dataset $X \in \Reals^{n \times m}$ from an online platform.  Let $\cE$ denote the index set of the observed entries, such that $X_{ij}$ is only nonzero for $(i,j) \in \cE$. The dataset is likely sparse as any given user has only interacted with a small subset of the entire product catalog, such that $|\cE| \ll nm$. We may be interested to predict user ratings for products that they haven't yet interacted with, as these predictions could be used for personalized recommendations. In particular we posit a ground truth ratings matrix $F \in \Reals^{n \times m}$, for which $\E[X_{ij}] = F_{ij}$ for observed entries. In addition to the ratings data, we assume access to side information in the form of covariates or features of the columns (but not the rows), which we denote $\{\beta_j\}_{j \in [m]}$. Such information asymmetry could arise due to product information being public, as opposed to user information being anonymized. For example, $\beta_j$ could represent features of the products which are publicly disclosed by the seller, whereas most users do not wish to disclose their personal attributes to the online platform.

Matrix completion, or matrix estimation, refers to the task of estimating $F$ from the noisy dataset $X$. Matrix estimation is a fundamental building block of standard data analysis pipelines in practice as part of the data cleaning stage, as most datasets in reality have measurement noise, mistakes, and missing data. The statistical properties of matrix estimation has been well studied in the context of low-rank models under anonymity, i.e. supposing that the only available data is the matrix $X$ itself, and no further attributes of the rows and columns are known.

In this paper, we additionally consider the impact of having access to one-sided covariate information, which is often available in real-world applications. Previous works in matrix estimation that consider access to side information assume that the side information reveals the row or column subspace of the ground truth matrix. In simple terms, this imposes that the data behaves linearly with respect to the revealed side information. In practice this is unrealistic, and there is often considerable effort put into heuristic feature engineering to generate a large set of functions of the covariates, with the hope that the generated features contain the desired subspace. As such, it is practically relevant to consider models of side information in which the subspace could be a nonlinear of the covariates. 

The results presented in this paper quantify the statistical gain for matrix estimation due to having access to one-side covariate information, under a nonparametric setting where the primary assumption is that the ground truth matrix is smooth with respect to the covariates. In particular, we assume that the ground truth matrix $F$ can be described by a latent function $f$ such that $F_{ij} = f(\alpha_i, \beta_j)$, where $\alpha_i$ are unknown latent features of the rows, and $\beta_j$ are the known column features. This nonparametric setting is a more expressive model class than low rank models. Our results show that there are three data regimes depending on the number of rows relative to the number of datapoints per row. In particular, when the number of rows are few, then estimating each row separately achieves the minimax optimal rate. When the number of rows are sufficiently large, we provide an algorithm which attains a more accurate estimate by sharing data across rows. Furthermore, when the number of rows are not too large, we show that our algorithm matches the minimax optimal rate for the oracle algorithm that has access to both row and column covariates, even though our algorithm does not have access to row covariates. 

\subsection{Related Literature}


The related work spans a wide literature across matrix estimation and nonparametric regression, both old and well-established fields of study. As a result, we do not present a comprehensive summary, but instead highlight the results that are most relevant to understanding the space in which our problem lies. We particularly draw attention to existing lower and upper bounds from nonparametric regression and matrix estimation.

\paragraph{Matrix Estimation.}
Classical sparse matrix estimation, also referred to as matrix completion, assumes that there are no observed covariates, and the only data available is a sparse noisy observation of a ground truth matrix. This problem has been widely studied for the setting when the ground truth matrix is low rank and incoherent, i.e. the latent factors exhibit regularity, and the observations are sampled uniformly across the matrix. Algorithms include nuclear norm minimization \cite{CandesPlan10,CandesTao10}, singular value thresholding \cite{mazumder2010tensor,Chatterjee15}, gradient descent \cite{KeshavanMontanariOh10b, GeLeeMa16}, alternating least squares \cite{jain2013low}, and nearest neighbor \cite{BorgsChayesLeeShah17, borgs2017iterative}. The performance is measured by the mean squared error (MSE) defined as
\[\MSE = \E\left[\frac{1}{mn}\sum_{i \in [n]}\sum_{j \in [m]} (\hat{F}_{ij} - F_{ij})^2\right],\]
where $\hat{F}$ denotes the estimate for the ground truth matrix $F$. For any rank $r$ matrix, low rank matrix completion algorithms produce estimates that converge in MSE as long as the fraction of observed entries $p = \Omega(r\min(n,m)^{-1} \log \min(n,m))$. This was shown to be tight up to polylog factors in \cite{KeshavanMontanariOh10b, CandesPlan10}.

In the low rank setting, there have been a sequence of works that consider additional side information in the form of covariate matrices or similarity based graphs. In inductive matrix completion, the primary assumption is that the covariate matrices reveal the row and column subspace of the ground truth matrix \cite{XuJinZhou13, zhong2015efficient, chiang2015matrix, eftekhari2018weighted, ghassemi2018global, chiang2018using, jain2013provable, lu2016sparse,guo2017convex, bertsimas2020fast, burkina2021inductive}. This implies that the ground truth matrix $F$ can be factored according to $YMZ^T$, where $Y \in \Reals^{n\times r_1}$ and $Z \in \Reals^{m \times r_2}$ are the given covariate matrices. As the unknown parameters reduces to only $r_1 \times r_2$, the given side information significantly reduces the sample complexity of matrix estimation from linear in $\max(n,m)$ to logarithmic in the dimension. Subsequent variations have been considered which allow for partial or noisy subspace information \cite{chiang2015matrix}. Assuming that the observed covariates reveal the subspace is a strong condition that is often not satisfied in practice, as it assumes the data is linearly related to the observed covariates. This assumption entirely bypasses the complexity of feature engineering, where people construct nonlinear functions of the covariates that attempt to heuristically guess features under which the model is linear.

An alternate form of side information has been considered in the form of graph side information or clustering based side information. The key idea of graph reglarized matrix completion is to impose a regularizer that encourages the estimate to be smooth with respect to an underlying graph \cite{zhou2012kernelized,kalofolias2014matrix, rao2015collaborative,yin2015laplacian,yu2016temporal, dong2019preconditioned,dong2021riemannian}. The majority of these works are primarily empirical with limited statistical guarantees. \cite{yu2016temporal,rao2015collaborative} provide bounds on the MSE, however their results show that the sample complexity still scales linearly in $\max(n,m)$, which does not clearly highlight the value or improvement theoretically of incorporating side information. Clustering based approaches use the graph side information to learn clusters, which they use fit a block constant matrix \cite{elmahdy2020matrix}. The results are also limited as they require the number of clusters to be small, and the ground truth matrix to be block constant. 


There has been a limited number of works that also consider a nonparametric model class as we assume in this paper. Under a Lipschitz model, \cite{song2016blind,li2020blind} require a significantly costlier sample complexity of $\min(n, m) \max(n,m)^{1/2}$. The models most relevant to our setting is from the graphon estimation literature, which specifically focuses on the case with binary observations and a symmetric matrix \cite{GaoLuZhou15, GaoLuMaZhou16, KloppTsybakovVerzelen15,Xu17}. When $n = m$, $\alpha_i = \beta_i \in [0,1]^d$, and $f$ is a symmetric $2d$-dimensional $(\lambda,L)$-Holder function, the singular value thresholding estimator achieves
\[\MSE = O((pm)^{-2\lambda/(2\lambda+d)}).\]
This matches the mimimax-optimal nonparametric rate for estimating a $d$-dimensional function given $pn$ observations, which would be the setting of estimating the latent function for a single row given the column covariates using only the datapoints within the row. Under the nonparametric setting, there has not been any existing works that incorporate side information with matrix estimation.


\paragraph{Nonparametric Regression.}
As we assume a nonparametric model, the crux of our algorithm will build upon nonparametric kernel regression. An excellent presentation of results and techniques in nonparametric estimation can be found in \cite{Tsybakov09}, of which we summarize a few key results below. 
Let the function class $\cF(\lambda,L)$ denote all $d$-dimensional $(\lambda,L)$-Holder functions with $\lambda \in (0,1]$ such that for all $x, x' \in [0,1]^d$ and $f \in \cF(\lambda,L)$,
\begin{align}
|f(x) - f(x')| \leq L \|x - x'\|_{\infty}^{\lambda}. \label{eq:Holder_bd}
\end{align}
Let $N$ denote the number of observed datapoints. Define the minimax MSE risk as 
\begin{align*}
 \cR^*_{N,2} &= \inf_{\hat{f}} \sup_{f \in \cF(\lambda,L)} \E_f\left[\int_{[0,1]^d} (\hat{f}(x) - f(x))^2 dx\right]
\end{align*}
The existing lower bounds show that 
\begin{align*}
\cR^*_{n,2} &= \Omega(N^{-2\lambda/(2\lambda + d)})
\end{align*}
These rates are achieved by locally polynomial estimators and thus it is tight. This literature however does not consider the value of sharing data amongst different regression tasks, as considered in our setting. If we performed regression on each row's data separately, then the minimax error rate would be $(pm)^{-2\lambda/(2\lambda+d)}$, as the number of datapoints in a given row is $N = pm$.

The technical aspects of our algorithm and proof rely upon the task of estimating the $L_2$ distance between two functions $f$ and $f'$ given observations from both. In particular we use the results from \cite{LepskiNemirovskiSpokoiny99} which shows that the minimax optimal rate for estimating $\|f\|_2$ with $N$ observations is 
\begin{align*}
\E[(\|\hat{f}\|_2 - \|f\|_2)^2] \asymp \max(N^{-4\lambda/(4\lambda + d)}, N^{-1/2}),
\end{align*}
which is faster than the previously stated estimation rate of $N^{-2\lambda/(2\lambda + d)}$. This is the inspiration of our algorithm, as we use the a similar procedure to estimate distance between a pair of rows.

\paragraph{Transfer Learning.}
Matrix estimation with one-sided covariates is also related to transfer learning. Given covariate features associated to columns, the task of prediction for a single row reduces to standard supervised learning or regression, as we can simply fit a function of the covariates to the observed data in this row. In our motivating example, this would be equivalent to separately learning for each user a personalized function mapping from covariates to a predicted rating, which is trained from the data collected from the user itself. A natural question is whether there is value to be gained from sharing data, i.e. using data from user $j$ to predict the rating function of user $i$.

Transfer learning refers to the challenge of learning from data in task A and ``transferring that knowledge'' to perform well in task B. The guarantees in the literature all rely on some minimal similarity between tasks A and B. This is expected as we should not expect to learn task B from training on task A if they are very different. In addition, the typical scenario of transfer learning assumes a very small number of tasks and asymmetric data requirements where there is a rich dataset for task A but a limited dataset for task B. 

The setting we consider is different as the number of ``tasks'' can be large as it corresponds to the number of rows, e.g. users in the platform. As such we shouldn't expect that all the rows, i.e. users, are similar. Due to anonymization of the users, we do not have a priori knowledge of how the ratings functions associated to different users are related to one another. Our algorithm addresses the challenge of determining which rows are similar to each other. 
\subsection{Contributions}

We present an algorithm for nonparametric matrix estimation exploiting one-sided covariate knowledge. The algorithm uses the dataset to estimate distances between rows, and subsequently computes nearest neighbor estimates using the estimated row distances and the column covariates to determine the nearest neighbors.

We provide theoretical guarantees for our algorithm. When the matrix is short and fat, in particular $n = O((mp)^{d_1/(2\lambda+d_2)})$, then the optimal estimator is to simply estimate each row separately. In this regime, the bias introduced by incorporating data from another row is larger than the accuracy to which a single row can already be estimated.

In a moderate regime when $n = \omega((mp)^{d_1/(2\lambda+d_2)})$ and $n = O\left(\left(mp\right)^{\min\left((2\lambda + d_1)/d_2, (2d_1 + d_2)/(4\lambda + d2)\right)}\right)$, our proposed algorithm matches the minimax-optimal nonparametric rate of the oracle algorithm which is given access to the row covariates. This implies that we in fact lose very little by only knowing the column covariates and not the row covariates as our algorithm performs as well as if it knew the row covariates. 

When the matrix is tall and narrow, in particular $n = \omega\left(\left(mp\right)^{\min\left((2\lambda + d_1)/d_2, (2d_1 + d_2)/(4\lambda + d2)\right)}\right)$, we show that our proposed algorithm still outperforms naive regression on each row separately, however it does not match the oracle algorithm that has knowledge of the row covariates. In this setting, the relatively few observations per row limits the accuracy to which the distance between rows can be estimated. Our algorithm improves upon the naive approach of estimating each row separately, as the distance between rows can be estimated more efficiently than estimating the full function for each row separately. The algorithm is limited by its ability to estimate distances between rows, so that the oracle algorithm which exploits row covariates will outperform when $n$ is large.

We provide simulations that show our algorithm outperforms other existing baselines, even compared to low rank matrix estimation algorithms on low rank data itself. Our empirical results highlight that in sparse data regimes, even when the true model is actually low rank, our algorithm which utilizes nonlinear side information at a cost of considering the larger class of nonparametric models outperforms low rank matrix completion algorithms.
\section{Model}

Consider a dataset consisting of a sparse data matrix $X \in \Reals^{n\times m}$ and observed column covariates $\{\beta_i\}_{i \in [n]}$. The goal is to estimate a ground truth matrix $F$ given the data matrix and observed covariates. We make the following assumptions on the data generating model.

\begin{assumption}[Row and column covariates]
Each row $u \in [n]$ is associated with a latent covariate $\alpha_u \in [0,1]^{d_1}$, and each column $i \in [m]$ is associated with an observed covariate $\beta_i \in [0,1]^{d_2}$. These covariates are sampled independently uniformly on the specified unit hypercubes, $\alpha_u \sim U([0,1]^{d_1})$ and $\beta_i \sim U([0,1]^{d_2})$. The column covariates $\{\beta_i\}_{i \in [m]}$ are observed, but the row covariates $\{\alpha_u\}_{u \in [n]}$ are not known.
\end{assumption}

\begin{assumption}[Gaussian observation noise]
Each observed datapoint $X_{ui}$ is a noisy signal of a ground truth function $f$, perturbed with additive Gaussian noise, 
\begin{align*}
X_{ui} = f(\alpha_u, \beta_i) + \epsilon_{ui},
\end{align*}    
where $\epsilon_{ui} \sim N(0,\sigma^2)$ are independent mean-zero Gaussian noise terms. For $F \in \Reals^{n \times m}$ denoting the ground truth matrix, it follows that $\E[X_{ui}] = F_{ui} = f(\alpha_u, \beta_i)$. 
\end{assumption}

\begin{assumption}[Smoothness of latent function]
The latent function $f$ is an $(\lambda,L)$-Holder function with $\lambda \in (0,1]$. For $\lambda = 1$, $f$ is $L$-Lipschitz. 
\end{assumption}

\begin{assumption}[Uniform Bernoulli sampling]
Each entry is observed independently with probability $p$. For $\cE$ denoting the set of observed indices, each index pair $(u,i) \in \cE$ with probability $p$. We overload notation and also let $\cE_{ui}$ denote the indicator function $\Ind((u,i) \in \cE)$.
\end{assumption}



We assume a nonparametric model, where the ground truth matrix is described by a latent function $f$. This is in contrast to the majority of the literature which assumes a low rank model. Any Lipchitz function can be approximated by an approximately low rank matrix \cite{Chatterjee15, UdellTownsend17, Xu17}. Any rank $r$ matrix can be described with the inner product function computed over $r$ dimensional latent feature spaces, which also exhibits Lipschitzness. When it comes to side information however, our model is more realistic as we allow for nonlinear relationships between the side information and the observed matrix data. In particular, the attempts to incorporate side information to low rank models require the side information to reveal the latent subspace, which is a significantly stronger assumption than mildly assuming smoothness as in our model.

\section{Algorithm}

Assume that we have two freshly sampled sets of observations, $\cE'$ used for learning the row distances, and $\cE''$ used for generating the final prediction. We can simply take the dataset and partition it uniformly at random into these two subsets with equal probability. For simplicity of analysis we will assume that we have in fact two freshly sampled datasets, however the result should hold for a single dataset partitioned via uniformly random sample splitting using a slight modification of the presented argument. In a practical implementation of our algorithm in the experiments section, we do not sample split but simply use the same samples for each part of the algorithm. Empirically our algorithm still performs well without sample splitting.

If we knew the row latent variables $\alpha$ in addition to the column latent variables $\beta$, then we can simply use any nonparametric regression estimator such as kernel regression to match the minimax optimal rates. The idea of kernel regression is simple; estimate the value of the target function at $(\alpha,\beta)$ using a weighted average of the datapoints, where higher weights are given to nearby or similar datapoints, as determined by the kernel. For simplicity we consider a rectangular kernel, which gives equal weight to all datapoints for which the corresponding $(\alpha',\beta')$ are at distance no more than a specified threshold. This is also equivalent to a fixed threshold nearest neighbor algorithm, where the nearest neighbor set is defined by the distances in the latent space. 

In our problem setting, we do not have knowledge of $\alpha$. Instead we propose an algorithm that uses data to estimate a proxy for distance between the rows, which is then used to determine the set of nearest neighbors used to construct the final estimates. As a result, the crux of the algorithm and resulting analysis is to make sure that the estimated distances are estimated closely enough to add value to the final estimates with respect to the bias variance tradeoff. We construct distances that approximate the $L_2$ difference in the latent functional space, evaluated with respect to the column latent variables, given by
\[d^2(u,v) = \frac{1}{m} \sum_{l \in [m]} \left( f(\alpha_u, \beta_l) - f(\alpha_v, \beta_l)\right)^2.\]
Since we don't have access to the latent function $f$, we instead estimate the function $f(\alpha_u,\cdot)$ associated to row $u$ using the data in row $u$ itself. While any nonparametric estimator could be used, we use the Nadaraya-Watson estimator with a rectangular kernel with respect to the infinity norm for ease of analysis \cite{Tsybakov09}. The same results will also likely hold for other standard kernels such as the Gaussian kernel. One could also consider local polynomial estimators, which would exploit higher order smoothness for $\lambda > 1$.

Our algorithm has three steps, which we detail below.

\paragraph{Step 1: Initial row latent function estimates.} 
For each row $u$, compute $\hat{f}(u,i)$ to approximate $f(\alpha_u,\beta_i)$ via the Nadaraya-Watson estimator on row $u$'s data, according to 
\begin{align*}
\hat{f}(u,i) &= \frac{1}{W_{ui}}\sum_{j \in [m]} \cE'_{uj} X_{uj} K\left(\tfrac{\beta_i - \beta_j}{h}\right)  \\
\quad\text{ for }\quad
&W_{ui} = \sum_{j \in [m]} \cE'_{uj} K\left(\tfrac{\beta_i - \beta_i}{h}\right),
\end{align*}
with bandwidth $h$ and kernel function $K(b) = \Ind(\|b\| \leq \frac12)$, where $\|\cdot\|$ denotes the infinity norm, $\|b\| = \max_l |b_l|$. 

\paragraph{Step 2: Pairwise row distance estimates.} For each pair of rows $u$ and $v$, compute $\hat{d}^2(u,v)$ to approximate $d^2(u,v)$ by comparing $\hat{f}(u,i)$ and $\hat{f}(v,i)$ across all $i \in [m]$, according to 
\begin{align*}
\hat{d}^2(u,v) = \frac{1}{m}\sum_{i \in [m]}(\hat{f}(u, i) - \hat{f}(v, i))^2 - \xi^2_{uv},
\end{align*}
for $\xi^2_{uv}$ chosen as below to offset the bias that arises from the squared terms involving the observation noise,
\begin{align*}
\xi^2_{uv} := \frac{\sigma^2 }{m} \sum_{i \in [m]} \sum_{l \in [m]} \left(\tfrac{\cE'_{ul}}{W_{ui}^2} + \tfrac{\cE'_{vl}}{W_{vi}^2}\right) K^2\left(\tfrac{\beta_l - \beta_i}{h}\right).
\end{align*}

\paragraph{Step 3: Nearest neighbor estimates.} For each index pair $(u,i)$, estimate $F_{ui}$ using a fixed radius nearest neighbor estimator,
\begin{align*}
\hat{F}_{ui} = \frac{\sum_{v \in \cN_1(u,\eta_1)} \sum_{j\in \cN_2(i,\eta_2)} X_{vj} \cE''_{vj}}{\sum_{v \in \cN_1(u,\eta_1)} \sum_{j\in \cN_2(i,\eta_2)} X_{vj} \cE''_{vj}},
\end{align*}
where the neighborhood sets are defined as
\begin{align*}
\cN_1(u,\eta_1) &:= \{v \in [n]: \hat{d}^2(u,v) \leq \eta_1^2\} \\
\cN_2(i,\eta_2) &:= \{j \in [m]: \|\beta_i -\beta_j\| \leq \eta_2\}
\end{align*}
for some chosen thresholds $\eta_1, \eta_2$.

Our algorithm has three tuning parameters, $h$, $\eta_1$, and $\eta_2$. The naive computational complexity bounds are $O(m |\cE|) = O(n m^2)$ for Step 1, $O(n^2 m)$ for Step 2, and $O(nm |\cE|) = O(n^2 m^2)$ for Step 3. The most costly step is computing the nearest neighbor estimates. This can be improved computationally via approximate nearest neighbor algorithms, or by computing a block constant estimate resulting from clustering using the distances. The same performance guarantees can be achieved with a reduced computational complexity by choosing the number of blocks or clusters appropriately.
\section{Theoretical Guarantees}

We quantify the gain due to side information by bounding the mean squared error achieved by our algorithm relative to naive row regression. When there are very few rows, i.e. $n = O((mp)^{d_1/(2\lambda+d_2)})$, then estimating each row separately using classical regression techniques will obtain the minimax rate of 
\[\MSE = O((mp)^{-2\lambda/(2\lambda + d_2)}).\]
Note that in this setting, even if the row covariates $\alpha$ were observed, the achieved MSE from performing regression on the full matrix dataset would be $O((pmn)^{-2\lambda/(2\lambda + d_1 + d_2)}$, which is worse than the MSE achieved from estimating on each row separately, as it ignores the additional structure in a matrix dataset which enforces that the covariates of all the datapoints are aligned along a grid corresponding to the rows and columns. This observation highlights that when there are only a few rows, each row is sufficiently different such that the bias from sharing data outweighs the benefits of variance reduction. As a result, our algorithm focuses on the regime where $n = \omega((mp)^{d_1/(2\lambda+d_2)})$.

We will choose the bandwidth of our initial row regression estimates according to
\[h = \Theta\left(\left(\frac{pm}{\log mn}\right)^{-\min(1/d_2, 2/(d_2+ 4\lambda))}\right).\]
\begin{theorem} \label{thm:main}
For  $n = \omega((mp)^{d_1/(2\lambda + d_2)})$ and $n = O\left(\left(mp\right)^{\min\left((2\lambda + d_1)/d_2, (2d_1 + d_2)/(4\lambda + d2)\right)}\right)$, our algorithm with $\eta_1 = \eta_2^{\lambda} = (pnm)^{-\lambda/(2\lambda + d_1 + d_2)}$ achieves rate 
\[\MSE = O\left((pmn)^{-2\lambda/(2 \lambda + d_1 + d_2)}\right).\] 
For $n = \omega\left(\left(mp\right)^{\min\left((2\lambda + d_1)/d_2, (2d_1 + d_2)/(4\lambda + d2)\right)}\right)$ our algorithm with $\eta_1 = 2h^{\lambda}$, $\eta_2 = h$ achieves rate
\[\MSE = O\left(\left(\frac{pm}{\log mn}\right)^{-\min(2\lambda/d_2, 4\lambda/(d_2+ 4\lambda))}\right).\]
\end{theorem}

In the regime that $n = \omega((mp)^{d_1/(2\lambda+d_2)})$ and $n = O\left(\left(mp\right)^{\min\left((2\lambda + d_1)/d_2, (2d_1 + d_2)/(4\lambda + d2)\right)}\right)$, our estimator in fact achieves the nonparametric minimax optimal rate of an oracle regression algorithm which observed both row and column covariates. This means that even without knowledge of the row covariates, the algorithm can learn the row distances from the data itself accurately enough to match the oracle.

For $n = \omega\left(\left(mp\right)^{\min\left((2\lambda + d_1)/d_2, (2d_1 + d_2)/(4\lambda + d2)\right)}\right)$, our algorithm still improves upon the performance of naive row regression, but does not match the oracle. This is expected as the knowledge of the row covariates is more powerful when there is a large number of rows relative to the datapoints in each row. Furthermore when the number of datapoints in each row is small relative to the number of rows, then the performance of our algoritm is limited by the rate of estimating $L_2$ distance between the latent functions associated to pairs of rows. Our analysis for our algorithm is nearly tight, as there is a matching (up to polylog factors) minimax lower bound of $(mp)^{-4\lambda/(4\lambda + d_2)}$ for estimating the $L_2$ norm of a function given fixed design, i.e. the $\beta$ covariates are evenly spaced. The term $\left(pm/\log mn\right)^{-2\lambda/d_2}$ arises from the randomness of the column covariates $\beta$.

Without access to side information, classical matrix estimation requires a sample complexity linear in $\max(n,m)$, i.e. $p = \omega(\min(n,m)^{-1})$ in order for a convergent estimator to exist, as any low rank matrix completion algorithm requires a growing number of observed entries in each row and column. In contrast, when we have access to column covariates, the sample complexity is only linear in $n$, i.e. $p = \omega(m^{-1})$. When $m$ grows faster than $n$, then this could significantly reduce sample complexity. In particular, when $m$ is large and $p = o(n^{-1})$, with high probability there will be columns which have zero entries observed. Given the column covariates however, we can simply predict the empty column using the prediction of other columns with similar covariates. This highlights the inherent difference in sample complexity with and without side information. 



\section{Proof Sketch}
As the final estimate is constructed using fixed radius nearest neighbor, the primary piece of the proof is to show that the estimated distances concentrate, as stated in Lemma \ref{lemma_abbrv:distance_conc}.
\begin{lemma}\label{lemma_abbrv:distance_conc}
	With prob $1 - o(1)$, for all $u,v \in [n]^2$,
	\begin{align*}
		|\hat{d}(u,v) -  d(u,v)| 
		&\leq O\left(\frac{2^{3d_2/4} \sigma \log^{1/2}(n) }{(pm)^{1/2} h^{d_2/4}} + 2 L (h/2)^{\lambda}\right) 
	\end{align*}
	for $h$ satisfying $h = o(\log^{-2/d_2}(n))$.
\end{lemma}

To prove Lemma \ref{lemma_abbrv:distance_conc}, we will separately bound the error due to the additive Gaussian observation noise, and the error due to the randomness in sampling $\{\beta_i\}_{i\in[m]}$ and the indices in $\cE'$.
Let $\tilde{f}(u,i)$ and $\tilde{d}^2(u,v)$ denote the hypothetical distances estimated from comparing the Nadaraya-Watson estimator on row $u$'s data assuming no observation noise,
\begin{align*}
\tilde{f}(u,i) &= \frac{\sum_{j \in \cE'_u} K\left(\frac{\beta_i - \beta_j}{h}\right) f(\alpha_u,\beta_j)}{W_{ui}} \\
\tilde{d}^2(u,v) &= \frac{1}{m} \sum_{l \in [m]} \left( \tilde{f}(u, l) - \tilde{f}(v, l)\right)^2.
\end{align*}
The proof of Lemma \ref{lemma_abbrv:distance_conc} bounds the error due to observation noise as captured by $|\hat{d}(u,v) -  \tilde{d}(u,v)|$ using concentration inequalities that exploit the Gaussian distribution of the additive observation noise terms. We bound the error due to the sparse sampling as captured by $|\tilde{d}(u,v) - d(u,v)|$ using Holder-smoothness and regularity of the sampling model for the column covariates and the location of the observed entries.

By approximating $(u,i)$ via averaging over nearby datapoints $(v,j)$ for which $\beta_i - \beta_j$ is small and $d(u,v)$ is small, we can bound the MSE using the Holder smoothness property of $f$, the construction of $d(u,v)$, and the uniformity of the sampling model. Our proof slightly deviates from the typical bias variance calculations common to nearest neighbor estimators, as $d(u,v)$ only estimates the $L_2$ difference between the latent functions of $u$ and $v$. A bound on the $L_2$ difference can translate into a very loose bound on the $L_{\infty}$ distance especially in high dimensions. As a result, $d(u,v)$ being small does not imply that $f(\alpha_u,\beta)$ is close to $f(\alpha_v,\beta)$ for all values of $\beta$. However as our goal is to compute an aggregate MSE bound, we can still obtain a good bound on the MSE without having a tight $L_{\infty}$ bound. In particular, under the good event that the distances estimates concentrate and the nearest neighborhood sets are sufficiently large, Lemma \ref{lemma_abbrv:nearest_neighbor} bounds the MSE of our estimator.

\begin{lemma} \label{lemma_abbrv:nearest_neighbor}
Conditioned on the events $\cap_{u,v\in [n]^2} \{|\hat{d}(u,v) - d(u,v)| \leq \Delta\}$, $\cap_{u\in [n]} \{|\cN_1(u,\eta_1)| \geq z_1\}$ and $\cap_{i\in [m]} \{|\cN_2(i,\eta_2)| \geq z_2\}$, with respect to the randomness in $\cE''$ and $\{X_{ab}\}_{(a,b) \in \cE''}$, for $p = \omega((z_1 z_2)^{-1})$, with probability $1 - o(1)$, it holds that
\begin{align*}
MSE &= O\left(\frac{\sigma}{\sqrt{p z_1 z_2}} + L^2 \eta_2^{2\lambda} + (\eta_1 + \Delta)^2\right).
\end{align*}
\end{lemma}

The final result then follows from bounding the probability that the good events are violated along with a proper choice of the parameters $h, \eta_1$, and $\eta_2$, which then also determine $\Delta$, $z_1$, and $z_2$.

\begin{figure*}[!ht]
     \centering
     \begin{subfigure}[b]{0.32\textwidth}
         \centering
         \includegraphics[width=\textwidth]{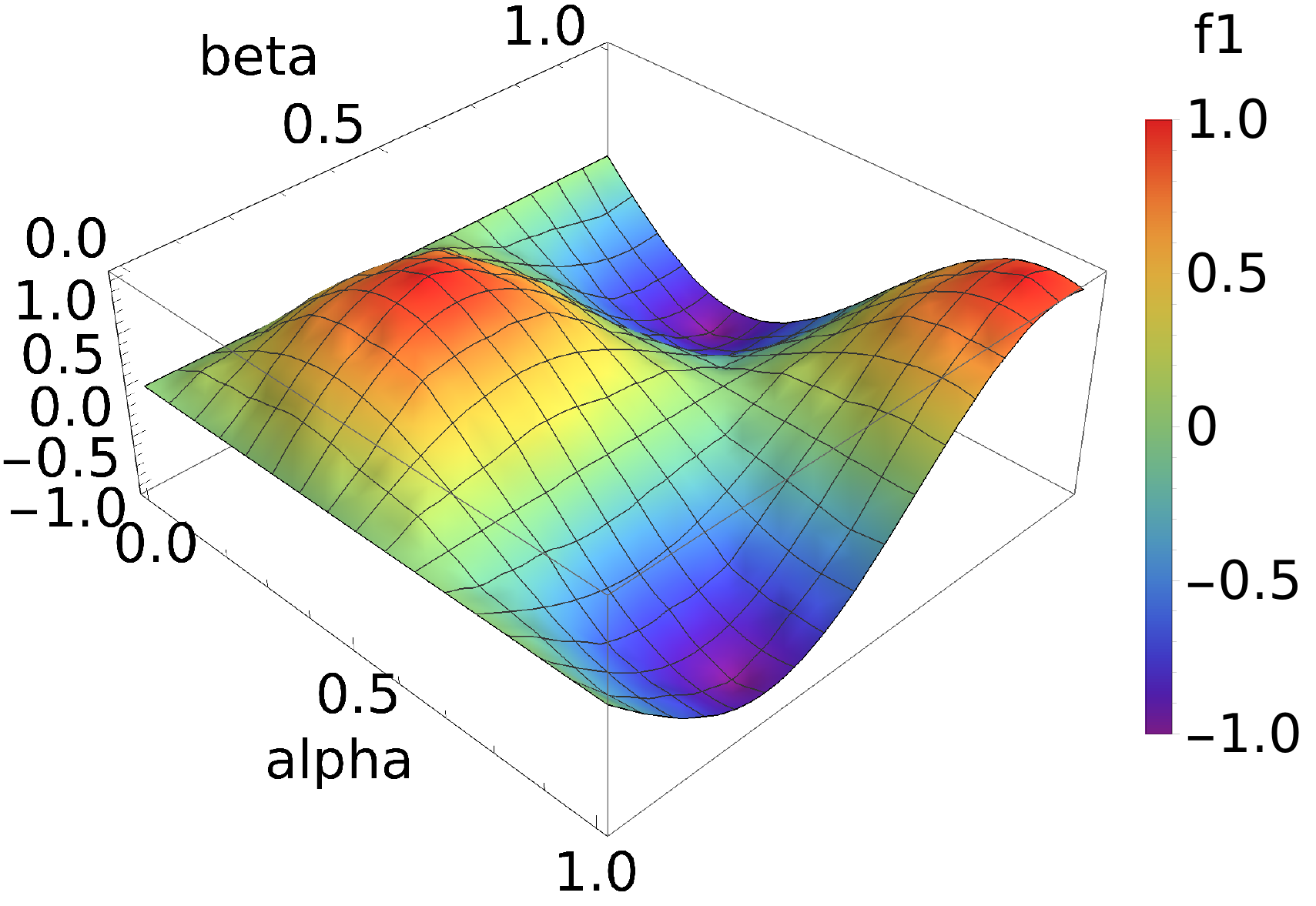}
         \label{fig:f1}
     \end{subfigure}
     \begin{subfigure}[b]{0.32\textwidth}
         \centering
         \includegraphics[width=\textwidth]{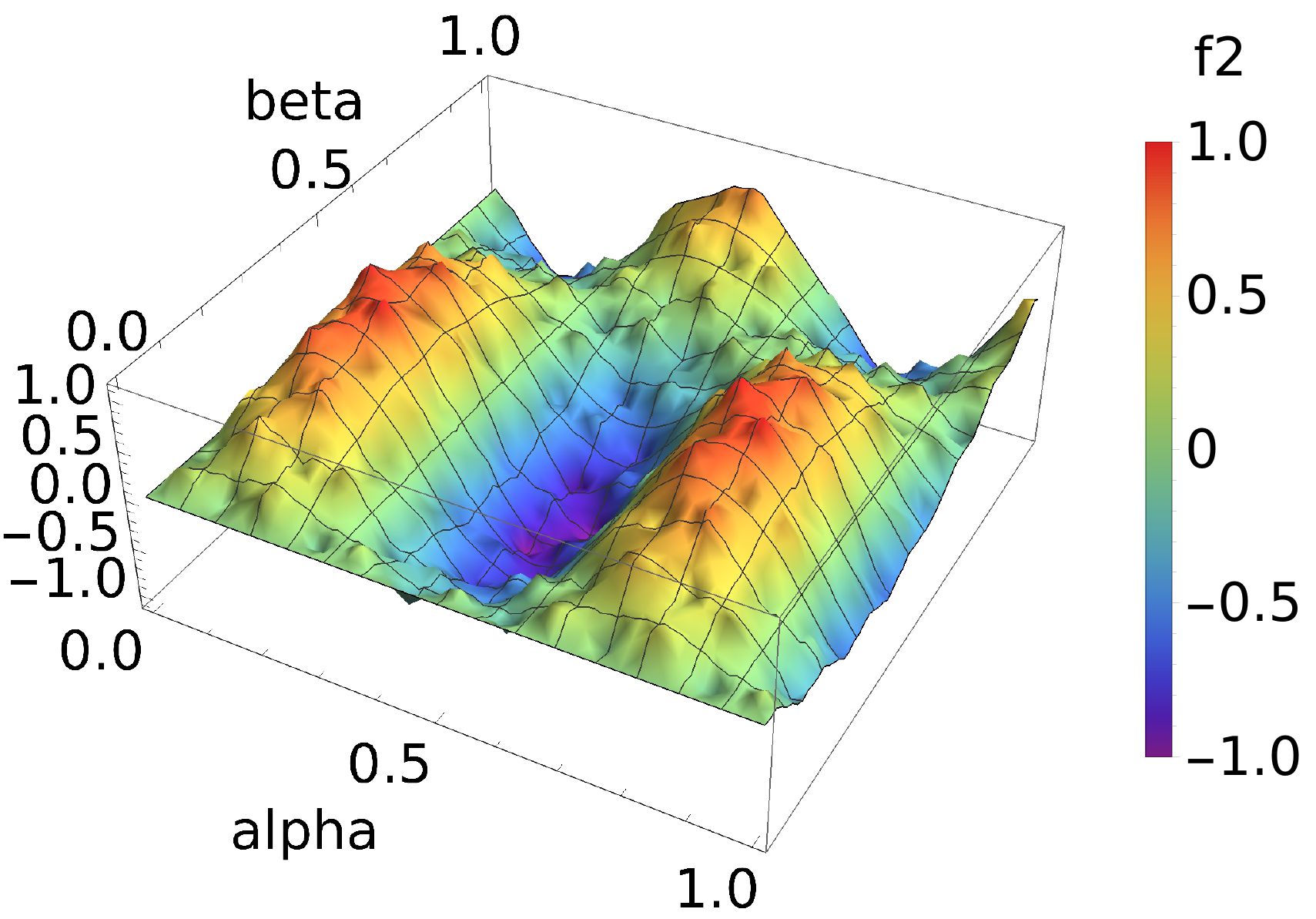}
         \label{fig:f2}
     \end{subfigure}
     \begin{subfigure}[b]{0.32\textwidth}
         \centering
         \includegraphics[width=\textwidth]{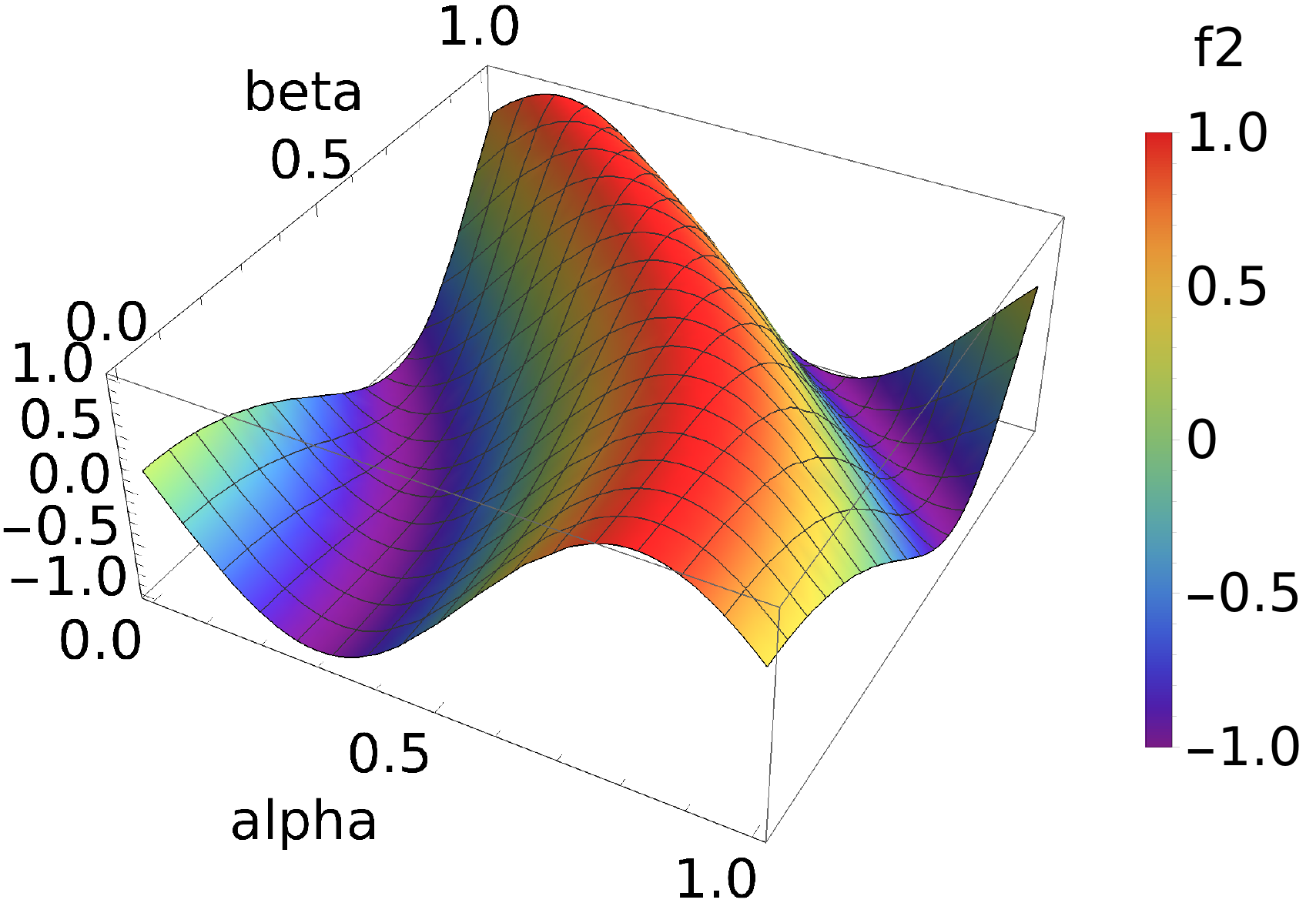}
         \label{fig:f3}
     \end{subfigure}
        \caption{Ground truth latent functions $f_1, f_2, f_3$ (left to right) used to generate the data.}
        \label{fig:truth}
\end{figure*}

\section{Synthetic Experiments}

We construct synthetic experiments to illustrate the performance of our algorithm in practice. We use the following three functions $f_1, f_2$, and $f_3$ defined below to generate the ground truth matrices for our experiments, where the row and column covariates are sampled from the uniform distribution on the unit interval. 
\begin{align*}
f_1(\alpha,\beta) &= \sin (5 \alpha) \sin (5 \beta) + 0.05 \left(\sin (25 \alpha) \sin (25\beta)\right)^3\\
f_2(\alpha,\beta) &= \sin (10 \alpha) \sin (4 \beta) + 0.2 \left(\sin (40 \alpha) \sin (40\beta)\right)^3\\
f_3(\alpha,\beta) &= \sin (3 + 6\alpha + 4\beta^2) 
\end{align*}
These functions are depicted in Figure \ref{fig:truth}. All of these functions can be written as a sum of two terms, which each can be written as a product of a function of $\alpha$ and a function of $\beta$. As a result, the corresponding ground truth matrices are all rank 2. The latent factors of the low rank decomposition however are non-linear functions of $\alpha$ and $\beta$, such that knowledge of $\alpha$ and $\beta$ does not reveal the latent row and column subspaces.

To construct each dataset, we sample the row covariates $\{\alpha_i\}_{i \in [n]}$ and column covariates $\{\beta_j\}_{j \in [m]}$ independently from $U[0,1]$. Each index $(i,j) \in [n]\times[m]$ is sampled independently with probability $p$, upon which the associated datapoint $X_{ij}$ is observed with additive noise distributed as $N(0,\sigma^2)$. The dataset consists of the column covariates $\{\beta_j\}_{j \in [m]}$ and the noisy matrix observations $\{X_{ij}\}_{(i,j) \in \cE}$. Importantly, $\{\alpha_i\}_{i \in [n]}$ is not part of the given dataset. 

We compare our algorithm both to the na\"ive algorithm that estimates each row separately using kernel regression, and to an oracle kernel regression algorithm that is given knowledge of both $\alpha$ and $\beta$. These algorithms both use the square kernel, to match the choice of kernel used in our algorithm. The kernel regression bandwidth parameters as well as the thresholds used for our nearest neighbor estimator are chosen using a grid search for values of $h$ and $\eta_2$ between $0.005$ and $0.2$ in increments of $0.005$. For tuning $\eta_1$, we select over the number of row nearest neighbors instead, which is also determined by a grid search between 1 and 50. We do not sample split in the implementation of our algorithm although it was assumed for the purposes of the analysis. As the ground truth matrix is rank 2, we also compared our algorithm against classical matrix completion algorithms such as alternating least squares and soft impute. Soft impute performed significantly worse by an order of magnitude, such that it is not displayed in the first row of plots in Figure \ref{fig:results}. We allow ALS to use the knowledge that the matrix is rank 2, but it does not use the knowledge of $\beta$ as there is no existing algorithm to incorporate side information when the latent features are nonlinear with respect to the observed covariates.
\begin{figure*}[!ht]
     \centering
     \includegraphics[width=\textwidth]{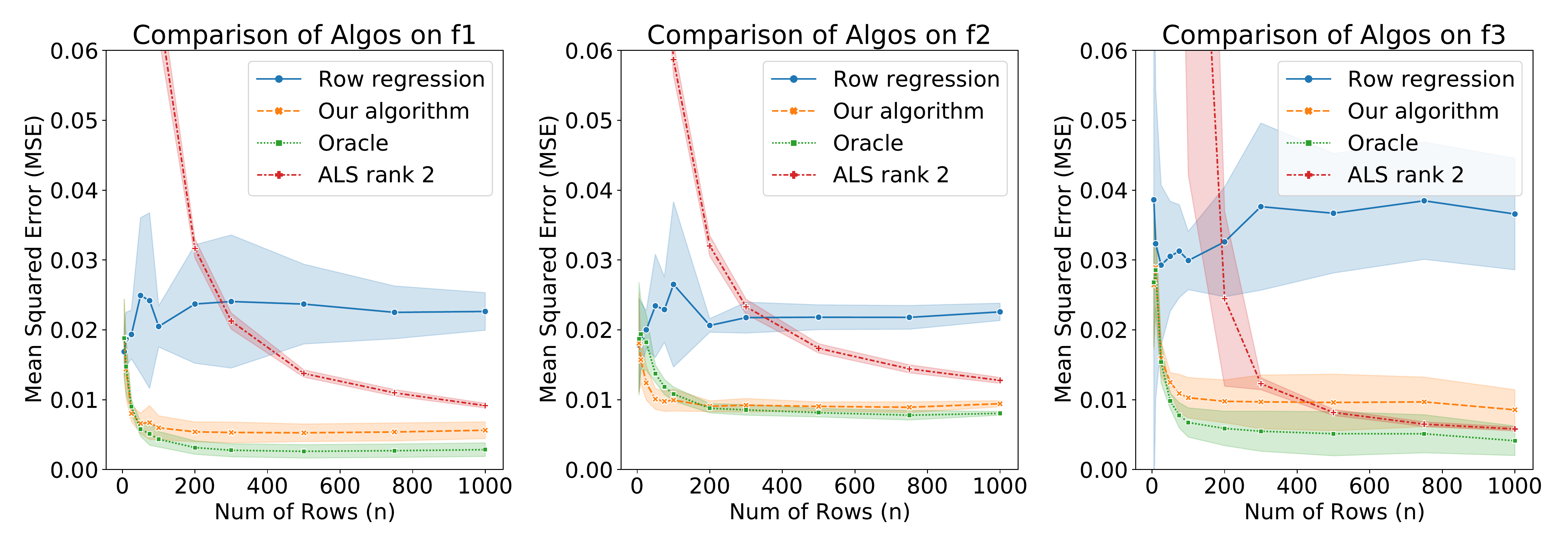} \\
     \includegraphics[width=\textwidth]{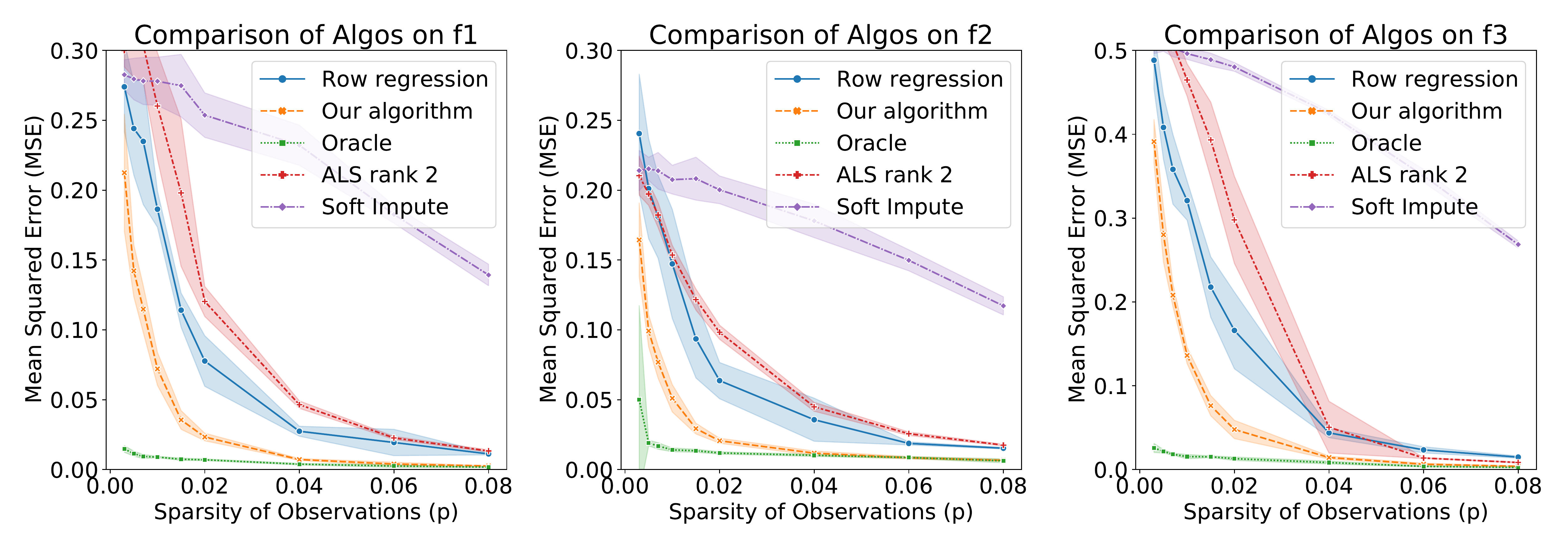} \\
        \caption{Comparison of our algorithm with row kernel regression, oracle kernel regression, alternating least squares (ALS) with rank 2, and Soft Impute. The top row shows the MSE of the resulting estimates as a function of the number of rows where $m = 500, p = 0.05$, and $\sigma = 0.2$. The bottom row shows the MSE of the resulting estimates as a function of the sparsity of observations in the dataset where $m = 500, n = 200$, and $\sigma = 0.2$. Note the scaling of the y-axis in the bottom row plots are slightly different to enhance readability.}
        \label{fig:results}
\end{figure*}



The first row of plots in Figure \ref{fig:results} shows the resulting MSE of each of the algorithms as a function of the number of rows $n$, where we set $m = 500$, $p = 0.05$, and $\sigma = 0.2$.
The second row of plots in Figure \ref{fig:results} shows the resulting MSE of each of the algorithms as a function of the sparsity of observations $p$, where we set $m = 500$, $n = 200$, and $\sigma = 0.2$.
For each combination of parameters, we generate 10 datasets, each of which is then given to each of the algorithms we benchmark.
The line plot shows the MSE averaged over the 10 sampled datasets, and the shaded region shows the standard deviation of the resulting MSE from these 10 datasets. The result across all three functions $f_1$, $f_2$, and $f_3$ are very similar. Our algorithm performs very well across all values of $n$, matching the oracle at small values of $n$ and still performing close to the oracle even at larger values of $n$. Our algorithm also performs well even at low levels of sparsity, significantly outperforming other benchmarks, and nearly matching the oracle at fairly low sparsity levels.

Perhaps more interesting is the comparison of our algorithm to classical matrix completion methods. In our examples, we chose functions that are in fact only rank 2, even though Lipschitz functions can have significantly higher rank. One might expect the matrix completion algorithms to perform well as they exploit the lower complexity low rank models as opposed to our nonparametric approach.  To the contrary, our results and simulations show that simple nonparmetric nearest neighbor style estimators can outperform low rank methods even in the class of low rank matrices, when the matrix is far from square, and when there is available covariate side information. As a point of reference, when $m = 500$, $n=200$, the minimum sparsity such that we observe 2 datapoints in each row and column on average would be $p = 0.02$, which is the rough threshold after which there is sufficient information to fit a rank 2 model. Note that our algorithm performs well at this level of extreme sparsity, outperforming row regression and classical matrix completion. At a sparsity level of $p=0.05$, we would expect that there is sufficient information to fit a low rank model at roughly $n = 40$. While our algorithm performs well at such low values of $n$, the matrix completion algorithms do not perform well until significantly higher values of $n$.

SoftImpute, a classical matrix completion algorithm based on nuclear norm minimization, performed very poorly in our experiments by an order of magnitude across all values of $n$, such that we did not display its results in the first row of plots shown above. When the number of rows is small, alternating least squares also performs poorly by an order of magnitude. This highlights that the role of the column covariates $\beta$ is especially critical when the number of rows is small, as ALS does not exploit this extra side information in contrast to the nonparametric methods. At a certain threshold of $n$, when the matrix becomes closer to square, ALS outperforms the naive row kernel regression, as ALS is able to use the rank 2 structure to share information across rows. Our algorithm remains competitive if not better than ALS even for larger values of $n$. As a function of the sparsity, it is also clear that SoftImpute and ALS perform significantly worse than our algorithm at low levels of sparsity.

%

\section*{Acknowledgements}

We gratefully acknowledge funding from the NSF under grants CCF-1948256 and CNS-1955997. Christina Lee Yu is also supported by an Intel Rising Stars Award.



\appendix
\section{Statement of Lemmas}

The final estimator using the rectangular kernel is equivalent to a fixed threshold nearest neighbor estimate, where the row distances are estimated via $\hat{d}$. As long as we can argue that the latent function behaves smoothly with respect to these estimated distances, then the result will follow from a simple bias and variance calculation. 

\begin{lemma}[Detailed Restatement of Lemma 5.1]
\label{lemma:distance_conc}
	With probability at least
	\[1 - m \exp\left(-\frac{\epsilon^2 m (h/2)^{d_2}}{2}\right) - m \exp\left(-\frac{\epsilon^2 m(2h)^{d_2}}{3}\right)
	- nm \exp\left(-\frac{\epsilon^2 m(h/2)^{d_2}(1-\epsilon) p}{2}\right) - 2 n^{2-\gamma} - 2n^{2-c \gamma},\]
	for all $u,v \in [n]^2$,
	\begin{align*}
		|\hat{d}(u,v) -  d(u,v)| 
		&\leq \frac{ 3 \cdot 2^{3d_2/4+1/2} \sigma \log^{1/2}(n^{\gamma}) (1+\epsilon)^{1/4}}{\sqrt{pmh^{d_2/2}}(1-\epsilon)} + 2 L (h/2)^{\lambda} 
	\end{align*}
	for $h$ satisfying $(2h)^{d_2} = o(\log^{-2}(n^{2\gamma}))$.
\end{lemma}

The proof of the above lemma will hinge on the following lemma.
\begin{lemma} \label{lemma:distSq_approx}
	Conditioned on the realizations of $\{\beta_i\}_{i\in[m]}$ and $\cE'$ such that the following good events hold:
	\begin{align*}
		\cE_1(q_1) &= \cap_{i\in [m]} \{|\cN_2(i, h)| \leq q_1\} \\
		\cE_2(q_2) &= \cap_{u,i\in [n]\times [m]} \{W_{ui} \geq q_2\},
	\end{align*}
	with respect to the randomness in the observations $\{X_{ab}\}_{(a,b) \in \cE'}$, for all $u,v$, there exists some constant $ c>0$ such that 
	\[\Prob(|\hat{d}^2(u,v) - \tilde{d}^2(u,v)| \geq t_1 + t_2) \leq 2 \exp\left(- \frac{t_1^2 m q_2}{4 q_1 \sigma^2 \tilde{d}^2(u,v)} \right) + 2 \exp\left(- c m \min\left(\frac{t_2^2 q_2^2}{4 \sigma^4 q_1}, \frac{t_2 q_2}{2\sigma^2 q_1}\right)\right). \]	
\end{lemma}

Next we need to argue that there are sufficiently many nearest neighbors. A lower bound on $|\cN_2(i,\eta_2)|$ for all $i$ follows from Chernoff's bound and union bound as we assumed in our model that $\beta_j \sim U([0,1]^{d_2})$. Under the event that the estimated distances concentrate, Lemma \ref{lemma:N1_conc} lower bounds $|\cN_1(u,\eta_1)|$ with high probability, resulting from the Holder smoothness property and the assumption that $\alpha_v \sim U([0,1]^{d_2})$.
\begin{lemma} \label{lemma:N1_conc}
	Conditioned on the realizations of $\{\beta_i\}_{i\in[m]}$, $\cE'$, and $\{X_{ab}\}_{(a,b) \in \cE'}$ such that the following good event holds:
	\begin{align*}
		\cE_3(\Delta) &= \cap_{u,v\in [n]^2} \{|\hat{d}(u,v) - d(u,v)| \leq \Delta\},
	\end{align*}
	with respect to the randomness in $\{\alpha_u\}_{u \in [n]}$, it holds that 
	\[\Prob\left(\forall u \left\{|\cN_1(u,\eta_1)| \geq n\left(\frac{\eta_1-\Delta}{L}\right)^{d_1/\lambda}(1-\epsilon)\right\}\right) \geq 1 - n \exp\left(-\frac{\epsilon^2 n}{2}\left(\frac{\eta_1-\Delta}{L}\right)^{d_1/\lambda}\right).\]
\end{lemma}



Under the good event that the distances estimates concentrate and the nearest neighborhood sets are sufficiently large, Lemma \ref{lemma:nearest_neighbor} bounds the MSE of our estimator.
\begin{lemma}[Detailed Restatement of Lemma 5.2]
\label{lemma:nearest_neighbor}
	Conditioned on the realizations of $\{\alpha_u\}_{u \in [n]}$, $\{\beta_i\}_{i\in[m]}$, $\cE'$, and $\{X_{ab}\}_{(a,b) \in \cE'}$ such that the following good events hold:
	\begin{align*}
		\cE_3(\Delta) &= \cap_{u,v\in [n]^2} \{|\hat{d}(u,v) - d(u,v)| \leq \Delta\}, \\
		\cE_4(z_1) &= \cap_{u\in [n]} \{|\cN_1(u,\eta_1)| \geq z_1\}, \\
		\cE_5(z_2) &= \cap_{i\in [m]} \{|\cN_2(i,\eta_2)| \geq z_2\}, \\
	\end{align*}
	with respect to the randomness in $\cE''$ and $\{X_{ab}\}_{(a,b) \in \cE''}$, it holds that
	\begin{align*}
		MSE &\leq 2 t + 4 L^2 \eta_2^{2\lambda} + \frac{4 (1 + \delta)(\eta_1 + \Delta)^2}{1 - \epsilon}
	\end{align*}
	with probability at least $1 - mn \exp\left(-\frac{\epsilon^2 p z_1 z_2}{2}\right) - 2 mn \exp\left(- \frac{t^2 (1-\epsilon) p z_1 z_2}{2 \sigma^2} \right) - mn\exp\left(-\frac{\delta^2 p z_2}{3}\right)$.
\end{lemma}

\section{Proof of Main Theorem}

The proof of the Main Theorem follows from putting together the results in the stated Lemmas along with choosing the right choices for the parameters $h$, $\eta_1$ and $\eta_2$.\\

\begin{proof}[Proof of Theorem 4.1]
We use Lemma \ref{lemma:distance_conc} to bound the probability of $\cE_3(\Delta)$ being violated by $3 (mn)^{1-\gamma} + 2 n^{2-\gamma} + 2n^{2-c \gamma}$ for 
\begin{align}
\Delta = \frac{ 3 \cdot 2^{3d_2/4+1/2} \sigma \log^{1/2}(n^{\gamma}) (1+o(1))}{\sqrt{pmh^{d_2/2}}} + 2 L (h/2)^{\lambda}, \label{eq:dist_Delta}
\end{align}
where we will choose $h$ such that $m p (h/2)^d_2 = \omega(1)$, and choose $\epsilon = o(\log^2(nm))$ such that $\epsilon^2 = \omega(\frac{2\gamma \log (nm)}{m p (h/2)^d_2})$.

We use Lemma \ref{lemma:N1_conc} to bound the probability of $\cE_4(z_1)$ being violated by $n^{1-\gamma}$ for 
\[z_1 = \max(1, n\left(\frac{\eta_1-\Delta}{L}\right)^{d_1/\lambda}(1-o(1))).\]
We will choose $\eta_1$ such that $n \left(\frac{\eta_1-\Delta}{L}\right)^{d_1/\lambda} = \omega(\log^2(n))$ and $\epsilon = o(1)$ such that $\epsilon^2 = \omega(\frac{2\gamma\log(n)}{n} \left(\frac{L}{\eta_1-\Delta}\right)^{d_1/\lambda})$. The max comes from the fact that $u$ is always a member of $\cN_1(u,\eta_1)$, such that $|\cN_1(u,\eta_1)|$ is always at least size 1.

By Chernoff's bound and union bound, as we assumed in our model that $\beta_j \sim U([0,1]^{d_2})$, the probability that event $\cE_5(z_2)$ will be violated is bounded above by $m^{1-\gamma}$ for 
\[z_2 = m\eta^{d_2}(1-o(1)),\]
where we choose $\eta_2$ such that $ m \eta_2^{d_2} = \omega(1)$ and $\epsilon = o(\log^2(m))$ such that $\epsilon^2 = \omega(\frac{2\gamma\log(m)}{m \eta_2^{d_2}})$.
	
As long as $p z_2 = \omega(\log^2(mn))$, then we can choose $\epsilon = o(1)$ and $\delta = o(1)$ satisfying $\epsilon^2 = \omega(\frac{2 \log(nm)}{p z_1 z_2})$ and $\delta^2 = \omega(\frac{3 \log(nm)}{p z_2})$ such that for the choice of 
\[t^2 = \frac{2\gamma \sigma^2 \log(mn)}{p z_1 z_2},\]
the error probability show in Lemma \ref{lemma:nearest_neighbor} is bounded above by $4mn^{1-\gamma}$. 

Putting it all together, for a sufficiently large constant $\gamma > 0$,
\begin{align*}
	MSE \leq  &\frac{4\gamma \sigma^2 \log(mn)}{p z_1 z_2} + 4 L^2 \eta_2^{2\lambda} + 4 (1 + o(1))(\eta_1 + \Delta)^2\\
	&+ 4mn^{1-\gamma}  +3 (mn)^{1-\gamma}+ 2 n^{2-\gamma} +2n^{2-c \gamma} +n^{1-\gamma}+m^{1-\gamma} \\
	=& O\left(\frac{\sigma^2 \log(mn)}{p z_1 z_2} + L^2 \eta_2^{2\lambda} + \eta_1^2 + \Delta^2\right) \\
	=& O\left(\frac{\sigma^2 \log(mn)}{p \max(1, n\left(\frac{\eta_1-\Delta}{L}\right)^{d_1/\lambda}) m\eta_2^{d_2}} + L^2 \eta_2^{2\lambda} + \eta_1^2 + \frac{ 2^{3d_2/2} \sigma^2 \log(n)}{pmh^{d_2/2}} + L^2 (h/2)^{2\lambda}\right)
\end{align*}
for
\begin{align}
	m p (h/2)^{d_2} &= \omega(\log^2(nm)) \\
	n \left(\frac{\eta_1-\Delta}{L}\right)^{d_1/\lambda} &= \omega(\log^2(n))) \\
	p m\eta_2^{d_2} &= \omega(\log^2(mn)).
\end{align}

The condition that $m p (h/2)^{d_2} = \omega(\log^2(nm))$ comes from the random design and the missing observations, i.e. we need that when we construct the initial row regression estimator $\hat{f}$, there are at least logarithmic in $n$ and $m$ datapoints contained in $W_{ui}$. If the values of the column covariates $\beta$ were evenly spaced, then this condition would not be necessary.
	
By choosing $h= \max\left(2 (\frac{pm}{\log mn})^{-1/d_2}, (\frac{L^2 p m }{\sigma^2 2^{3 d_2/2 + 2\lambda - 2} \log(mn)})^{-2/(d_2+ 4\lambda)}\right)$ such that 
\[\Delta = \Theta\left(\max\left(\left(\frac{pm}{\log mn}\right)^{-\lambda/d_2}, \left(\frac{L^2 p m }{\sigma^2 2^{3 d_2/2 + 2\lambda - 2} \log(mn)}\right)^{-2\lambda/(d_2+ 4\lambda)}\right)\right),\]
we arrive at our final MSE bounds.
	
{\bf Case 1:} $n = O((mp)^{d_1/(2\lambda + d_2)})$ then estimate each row separately.
\[MSE = O((mp)^{-2\lambda/(2\lambda + d_2)})\]

{\bf Case 2:} $n = O((mp)^{\min(\frac{2\lambda + d_1}{d_2},\frac{2d_1 + d_2}{4\lambda + d2})})$ and $n = \omega((mp)^{d_1/(2\lambda + d_2)})$, 
set $\eta_1 = \eta_2^{\lambda} = (pnm)^{-\lambda/(2\lambda + d_1 + d_2)}$ such that 
\[MSE = O((pnm)^{-2\lambda/(2\lambda + d_1 + d_2)})\]

{\bf Case 3:} $n = \omega((mp)^{\min(\frac{2\lambda + d_1}{d_2},\frac{2d_1 + d_2}{4\lambda + d2})})$,
set $\eta_1 = 2\Delta$, $\eta_2 = \Delta^{1/\lambda}$ such that
\[MSE = O(\Delta^2) = O\left(\max\left(\left(\frac{pm}{\log mn}\right)^{-2\lambda/d_2}, \left(\frac{L^2 p m }{\sigma^2 2^{3 d_2/2 + 2\lambda - 2} \log(mn)}\right)^{-4\lambda/(d_2+ 4\lambda)}\right)\right).\]
\end{proof}
\section{Concentration of Distance Estimates}


\begin{proof}[Proof of Lemma \ref{lemma:distance_conc}]
Recall that $\tilde{f}(u,i)$ and $\tilde{d}^2(u,v)$ denote the hypothetical distances estimated from comparing the Nadaraya-Watson estimator on row $u$'s data assuming no observation noise,
\begin{align*}
\tilde{f}(u,i) &= \frac{\sum_{j \in \cE'_u} K\left(\frac{\beta_i - \beta_j}{h}\right) f(\alpha_u,\beta_j)}{W_{ui}} \\
\tilde{d}^2(u,v) &= \frac{1}{m} \sum_{l \in [m]} \left( \tilde{f}(u, l) - \tilde{f}(v, l)\right)^2.
\end{align*}

First we show that 
\begin{align}
|\hat{d}^2(u,v) - \tilde{d}^2(u,v)| \leq a \tilde{d}(u,v) + b 
~~\implies~~ |\hat{d}(u,v) -  \tilde{d}(u,v)| \leq 3 \max(a, \sqrt{b}). \label{eq:dist_sq}
\end{align}
We prove this statement in two cases.\\
{\bf Case 1: } If $\tilde{d}(u,v) \leq \sqrt{b}$, then 
\begin{align*}
	|\hat{d}(u,v) -  \tilde{d}(u,v)| &\leq |\hat{d}(u,v)| + |\tilde{d}(u,v)| \\
	&\leq (|\hat{d}^2(u,v) - \tilde{d}^2(u,v)| + \tilde{d}^2(u,v))^{1/2} + |\tilde{d}(u,v)| \\
	&\leq \left(a \sqrt{b} + 2b\right)^{1/2} + \sqrt{b} \\
	&\leq (1 + \sqrt{3}) \max(a, \sqrt{b})
\end{align*}

{\bf Case 2: } If $\tilde{d}(u,v) \geq \sqrt{b}$, then
\begin{align*}
	|\hat{d}(u,v) - \tilde{d}(u,v)| 
	&= \frac{|\hat{d}^2(u,v) -  \tilde{d}^2(u,v)|}{\hat{d}(u,v) + \tilde{d}(u,v)} \\
	&\leq \frac{|\hat{d}^2(u,v) -  \tilde{d}^2(u,v)|}{\tilde{d}(u,v)} \\
	&\leq a + \sqrt{b} \\
	&\leq 2 \max(a, \sqrt{b})
\end{align*}

By Chernoff's bound and the model assumption that $\beta_j \sim U([0,1]^{d_2})$, with probability at least
\[1 - m \exp\left(-\frac{\epsilon^2 m (h/2)^{d_2}}{2}\right) - m \exp\left(-\frac{\epsilon^2 m(2h)^{d_2}}{3}\right),\]
it holds that for all $i \in [m]$, 
\[|\cN_2(i, h/2)| \geq z = m(h/2)^{d_2}(1-\epsilon)
~~\text{ and }~~
|\cN_2(i, h)| \leq q_1 = m(2h)^{d_2}(1+\epsilon).\]
Conditioned on the $\beta$ latent variables, $W_{ui}$ is a Binomial($|\cN_2(i, h/2)|,p$) random variable, such that by Chernoff's bound and union bound, with probability at least 
\[1-nm \exp\left(-\frac{\epsilon^2 |\cN_2(i, h/2)| p}{2}\right)\]
$W_{ui} \geq q_2 = |\cN_2(i, h/2)| p (1-\epsilon)$ for all $u,i$.

By combining these good events with Lemma \ref{lemma:distSq_approx}, it follows that with probability at least 
\[1 - m \exp\left(-\frac{\epsilon^2 m (h/2)^{d_2}}{2}\right) - m \exp\left(-\frac{\epsilon^2 m(2h)^{d_2}}{3}\right)
- nm \exp\left(-\frac{\epsilon^2 m(h/2)^{d_2}(1-\epsilon) p}{2}\right) - 2 n^{2-\gamma} - 2n^{2-c \gamma},\]
it holds that 
$|\hat{d}^2(u,v) - \tilde{d}^2(u,v)| \leq a \tilde{d}(u,v) + b$ for all $u,v \in [n]^2$
for the choice of 
\begin{align*}
a &=2 \sigma \log(n^{\gamma}) \sqrt{ \frac{q_1}{m q_2}} 
= 2^{d_2 + 1} \log(n^{\gamma})\sqrt{ \frac{\sigma^2 (1+\epsilon) }{m p (1-\epsilon)^2}}  \\
b &= \frac{ 2 \sigma^2 \log(n^{\gamma}) \sqrt{q_1} }{q_2 \sqrt{m}} 
= \frac{ 2^{3d_2/2+1} \sigma^2 \log(n^{\gamma}) \sqrt{(1+\epsilon)} }{pmh^{d_2/2}(1-\epsilon)^2}.
\end{align*}
By the condition that $h$ satisfies $(2h)^{d_2} = o(\log^{-2}(n^{2\gamma}))$, we can verify that for this choice of $b$, the first term within the $\min$ expression of the probability bound in Lemma \ref{lemma:distSq_approx} is the limiting term.

By the statement showed in \eqref{eq:dist_sq}, it follows that
\begin{align*}
|\hat{d}(u,v) -  \tilde{d}(u,v)| 
&\leq 3 \max(a,\sqrt{b}) \\
&= 3 \max\left(2^{d_2 + 1} \log(n^{\gamma})\sqrt{ \frac{\sigma^2 (1+\epsilon) }{m p (1-\epsilon)^2}},\frac{ 2^{3d_2/4+1/2} \sigma \log^{1/2}(n^{\gamma}) (1+\epsilon)^{1/4}}{\sqrt{pmh^{d_2/2}}(1-\epsilon)}\right).
\end{align*}
The second expression in the $\max$ will dominate asymptotically as long as $(2h)^{d_2} = o(\log^{-2}(n^{2\gamma}))$.  

Finally we use the Holder-smoothness to bound $|\tilde{d}(u,v) - d(u,v)|$. 
By triangle inequality, $|\|a\| - \|b\|| \leq \|a+b\|$ such that
\begin{align*}
&|\tilde{d}(u,v) - d(u,v)| \leq \left( \frac{1}{m} \sum_{l \in [m]} \left( \tilde{f}(\alpha_u, \beta_l) - f(\alpha_u, \beta_l) + f(\alpha_v, \beta_l) - \tilde{f}(\alpha_v, \beta_l)\right)^2 \right)^{1/2}
\end{align*}
By the Holder-smoothness conditions, for all $u, l$, 
\begin{align*}
|\tilde{f}(\alpha_u, \beta_l) - f(\alpha_u, \beta_l)| 
&= \left| \frac{1}{W_{ul}} \sum_{i \in \cE'_u} K\left(\frac{\beta_i - \beta_l}{h}\right) (f(\alpha_u, \beta_i) - f(\alpha_u, \beta_l)) \right| \\
&\leq \frac{\sum_{i \in \cE'_u} \Ind\left(\|\beta_i - \beta_l\| \leq \frac{h}{2}\right) \left|f(\alpha_u, \beta_i) - f(\alpha_u, \beta_l) \right|}{\sum_{i \in \cE'_u} \Ind\left(\|\beta_i - \beta_l\| \leq \frac{h}{2}\right)}  \\
&\leq \frac{\sum_{i \in \cE'_u} \Ind\left(\|\beta_i - \beta_l\| \leq \frac{h}{2}\right) L \|\beta_i - \beta_l\|^{\lambda}}{\sum_{i \in \cE'_u} \Ind\left(\|\beta_i - \beta_l\| \leq \frac{h}{2}\right)} \\
&\leq L (h/2)^{\lambda}.
\end{align*}
It follows then that 
$|\tilde{d}(u,v) - d(u,v)| 
\leq 2 L (h/2)^{\lambda}$.

\end{proof}


\begin{proof}[Proof of Lemma \ref{lemma:distSq_approx}]
	
	First we rearrange the desired expression into a term that scales linearly in the noise and a term that scales quadratically with the noise.
	\begin{align}
		&\hat{d}^2(u,v) - \tilde{d}^2(u,v) \\
		&= \frac{1}{m}\sum_{i \in [m]}\left((\hat{f}(\alpha_u, \beta_i) - \hat{f}(\alpha_v, \beta_i))^2 - (\tilde{f}(\alpha_u, \beta_i) - \tilde{f}(\alpha_v, \beta_i))^2\right) - \xi_{uv}^2  \\
		&= \frac{1}{m}\sum_{i \in [m]}\left(\tilde{f}(\alpha_u, \beta_i) - \tilde{f}(\alpha_v, \beta_i) + \sum_{l \in [m]} K\left(\frac{\beta_i - \beta_l}{h}\right) \left(\frac{\epsilon_{ul}\Ind(l \in \cE'_u)}{W_{ui}}- \frac{\epsilon_{vl}\Ind(l \in \cE'_v)}{W_{vi}}\right)\right)^2 \\
		&\qquad - \frac{1}{m}\sum_{i \in [m]}(\tilde{f}(\alpha_u, \beta_i) - \tilde{f}(\alpha_v, \beta_i))^2 - \xi_{uv}^2  \\
		&=  \frac{1}{m}\sum_{i \in [m]}\left(\sum_{l \in [m]} K\left(\frac{\beta_i - \beta_l}{h}\right) \left(\frac{\epsilon_{ul} \Ind(l \in \cE'_u)}{W_{ui}} - \frac{\epsilon_{vl} \Ind(l \in \cE'_v)}{W_{vi}}\right)\right)^2 \label{eq1a} \\
		&\quad + \frac{1}{m}\sum_{i \in [m]} (\tilde{f}(\alpha_u, \beta_i) - \tilde{f}(\alpha_v, \beta_i)) \sum_{l \in [m]} K\left(\frac{\beta_i - \beta_l}{h}\right) \left(\frac{\epsilon_{ul} \Ind(l \in \cE'_u)}{W_{ui}} - \frac{\epsilon_{vl} \Ind(l \in \cE'_v)}{W_{vi}}\right) \label{eq1b} \\
		&\quad - \xi_{uv}^2. \label{eq1c}
	\end{align}
	
	The expression in \eqref{eq1b} is a linear combination of mean zero independent Gaussian random variables, thus it is also normally distributed and mean zero. We use the inequality that $2\Cov(XY) \leq \Var(X) + \Var(Y)$ to bound the variance of the expression in \eqref{eq1b}. Let us denote 
	\[Z_{il} = (\tilde{f}(\alpha_u, \beta_i) - \tilde{f}(\alpha_v, \beta_i)) \sum_{l \in [m]} K\left(\frac{\beta_i - \beta_l}{h}\right) \left(\frac{\epsilon_{ul} \Ind(l \in \cE'_u)}{W_{ui}} - \frac{\epsilon_{vl} \Ind(l \in \cE'_v)}{W_{vi}}\right) .\]
	By the construction of the kernel function $K$, $Z_{il}$ is zero if $\|\beta_i - \beta_l\| > \frac{h}{2}$, and thus $\Cov(Z_{il}, Z_{jl})$ is zero if $\|\beta_i - \beta_j\| > h$. Then
	\begin{align*}
		\Var[\eqref{eq1b}] 
		&= \Var\left[\frac{1}{m}\sum_{i \in [m]} Z_{il} \right]\\
		&= \frac{1}{m^2} \sum_{(i,j) \in [m]^2} \Cov(Z_{il}, Z_{jl}) \\
		&\leq \frac{1}{m^2} \sum_{(i,j) \in [m]^2} \Ind(|\beta_i - \beta_j| < h) \frac{1}{2} \left(\Var(Z_{il}) + \Var(Z_{jl})\right) \\
		&= \frac{1}{m^2} \sum_{i \in [m]} \Var(Z_{il}) \sum_{j \in [m]} \Ind(|\beta_i - \beta_j| < h) \\
		&\leq \frac{q_1}{m^2} \sum_{i \in [m]} \Var(Z_{il})
	\end{align*}
	where the last step follows from conditioning on $\cE_1(q_1)$.
	
	Let's compute a bound on $\Var(Z_{il})$.
	\begin{align*}
		&\Var[Z_{il}] \\
		&\leq (\tilde{f}(\alpha_u, \beta_i) - \tilde{f}(\alpha_v, \beta_i))^2 \sigma^2 \sum_{l \in [m]} K^2\left(\frac{\beta_i - \beta_l}{h}\right) \left(\frac{\Ind(l \in \cE'_u)}{W_{ui}^2} + \frac{\Ind(l \in \cE'_v)}{W_{vi}^2}\right) \\
		&= (\tilde{f}(\alpha_u, \beta_i) - \tilde{f}(\alpha_v, \beta_i))^2 \sigma^2 \sum_{l \in [m]} \Ind\left(\|\beta_i - \beta_l\|_{\infty} \leq \frac{h}{2}\right) \left(\frac{\Ind(l \in \cE'_u)}{W_{ui}^2} + \frac{\Ind(l \in \cE'_v)}{W_{vi}^2}\right) \\
		&= (\tilde{f}(\alpha_u, \beta_i) - \tilde{f}(\alpha_v, \beta_i))^2 \sigma^2 \left(\frac{1}{W_{ui}} + \frac{1}{W_{vi}}\right)\\
		&\leq (\tilde{f}(\alpha_u, \beta_i) - \tilde{f}(\alpha_v, \beta_i))^2 \sigma^2 \left(\frac{2}{q_2}\right)
	\end{align*}
	where the last step follows from conditioning on $\cE_2(q_2)$.
	
	Putting it together, 
	\begin{align*}
		\Var[\eqref{eq1b}] 
		&\leq \frac{2 q_1 \sigma^2}{m^2 q_2} \sum_{i \in [m]} (\tilde{f}(\alpha_u, \beta_i) - \tilde{f}(\alpha_v, \beta_i))^2 \\
		&= \frac{2 q_1 \sigma^2 \tilde{d}^2(u,v)}{m q_2} 
	\end{align*}
	
	Conditioned on $\alpha$, $\beta$, and $\cE'$, the expression \eqref{eq1b} is normally distributed such that by Hoeffding's inequality,
	\begin{align*}
		\Prob(|\eqref{eq1b}| \geq t) 
		&\leq 2 \exp\left(- \frac{t^2 m q_2}{4 q_1 \sigma^2 \tilde{d}^2(u,v)} \right).
	\end{align*}

	Next we bound the expression in \eqref{eq1a}; it can be written as $y^T y$ for $y = Q x$, where $x$ is the $2m$-dimensional vector of additive noise terms
	\[x_l = \begin{cases}
		\epsilon_{ul} \text{ if } l \in [m]\\
		\epsilon_{v,l-m} \text{otherwise}
	\end{cases}\]
	and $Q$ is the $m \times 2m$ scaled kernel matrix
	\begin{align*}
		Q_{il} &= \begin{cases}
			K\left(\frac{\beta_l - \beta_i}{h}\right) \frac{\Ind((u,l) \in \cE')}{W_{ui}} &\text{ if } l \in [m] \\
			K\left(\frac{\beta_{l-m} - \beta_i}{h}\right) \frac{\Ind((v,l-m) \in \cE')}{W_{vi}} &\text{ otherwise }
		\end{cases}.
	\end{align*}
	The vector $x$ has mean 0 and a diagonal covariance matrix $\sigma^2 I_{2m}$, where $I_{2m}$ denotes the identity matrix. We can verify that the expected value of $y^T y$ is equal to $m\xi^2_{uv}$,
	\begin{align*}
		\E[y^T y] &= \sigma^2 \text{Tr}(Q Q^T) \\
		&= \sigma^2  \sum_{i \in [m]} \sum_{l \in [m]} Q_{il}^2 \\
		&= \sigma^2 \sum_{i \in [m]} \sum_{l \in [m]} \frac{\Ind((u,l) \in \cE') K^2\left(\frac{\beta_l - \beta_i}{h}\right)}{W_{ui}^2} + \sigma^2 \sum_{i \in [m]} \sum_{l \in [m]} \frac{\Ind((v,l) \in \cE') K^2\left(\frac{\beta_l - \beta_i}{h}\right)}{W_{vi}^2} \\
		&= m \xi^2_{uv}.
	\end{align*}
	We will apply the Hanson-Wright concentration inequality, which involves computing bounds on the Frobenius norm and spectral norm of $Q Q^T$.  Let us first obtain a bound for $[Q Q^T]_{ij}$ (we can drop the absolute values because all the involved terms are nonnegative),
	\begin{align*}
		\left|[Q Q^T]_{ij}\right| &= \left|\sum_{l \in [m]} Q_{il} Q_{jl}\right| \\
		&= \sum_{l \in [m]} \left(\frac{\Ind((u,l) \in \cE')}{W_{ui} W_{uj}} + \frac{\Ind((v,l) \in \cE')}{W_{vi} W_{vj}}\right) K\left(\frac{\beta_l - \beta_i}{h}\right) K\left(\frac{\beta_l - \beta_j}{h}\right) \\
		&= \sum_{l \in [m]} \left(\frac{\Ind((u,l) \in \cE')}{W_{ui} W_{uj}} + \frac{\Ind((v,l) \in \cE')}{W_{vi} W_{vj}}\right) \Ind\left(\|\beta_l - \beta_i\| \leq \frac{h}{2}\right) \Ind\left(\|\beta_l - \beta_j\| \leq \frac{h}{2}\right) \\
		&\leq  \sum_{l \in [m]} \left(\frac{\Ind((u,l) \in \cE')}{W_{ui} W_{uj}} + \frac{\Ind((v,l) \in \cE')}{W_{vi} W_{vj}}\right) \Ind\left(\|\beta_l - \beta_i\| \leq \frac{h}{2}\right) \Ind\left(\|\beta_i - \beta_j\| \leq h\right) \\
		&= \Ind\left(\|\beta_i - \beta_j\| \leq h\right) \left(\frac{1}{W_{uj}} + \frac{1}{W_{vj}}\right) \\
		&\leq \frac{2}{q_2} \Ind\left(\|\beta_i - \beta_j\| \leq h\right)
	\end{align*}
	where the last inequality follows from conditioning on $\cE_2(q_2)$.
	
	We can use the entrywise bound to obtain an upper bound on the Frobenius norm,
	\begin{align*}
		\|Q Q^T\|_F 
		&= \left(\sum_{i \in [m]} \sum_{j \in [m]} [Q Q^T]_{ij}^2\right)^{1/2} \\
		&\leq \left(\sum_{i \in [m]} \sum_{j \in [m]}\Ind\left(\|\beta_i - \beta_j\| \leq h\right) \left(\frac{2}{q_2} \right)^2 \right)^{1/2} \\
		&\leq \frac{2 (m q_1 )^{1/2} }{q_2}
	\end{align*}
	where the last inequality holds by conditioning on $\cE_1(q_1)$.
	
	By symmetry, $\|QQ^T\|_1 = \|QQ^T\|_{\infty}$. By Holder's inequality,
	\begin{align*}
		\|Q Q^T\|_2 
		&\leq \sqrt{\|QQ^T\|_1 \|QQ^T\|_{\infty}} \\
		&= \max_j \sum_{i}\left|[Q Q^T]_{ij}\right|\\
		&\leq \max_j \sum_{i} \frac{2}{q_2} \Ind\left(\|\beta_i - \beta_j\| \leq h\right) \\
		&= \frac{2q_1}{q_2}
	\end{align*}
	
	We apply the Hanson-Wright inequality to bound $|\eqref{eq1a} - \xi^2_{uv}|$ with high probability,
	\begin{align*}
		\Prob\left(\left|\eqref{eq1a} - \xi^2_{uv}\right| \geq t\right) 
		&\leq 2 \exp\left(- c m \min\left(\frac{t^2 q_2^2}{4 \sigma^4 q_1}, \frac{t q_2}{2\sigma^2 q_1}\right)\right).
	\end{align*}
	for a constant $c>0$. The final result follows from applying union bound.
\end{proof}

\section{Nearest Neighbor Analysis}

%


\begin{proof}[Proof of Lemma \ref{lemma:N1_conc}]
Conditioned on $|\hat{d}(u,v) - d(u,v)| \leq \Delta$ for all $u,v$,
It follows that
\[|\cN_1(u,\eta_1)| \geq \sum_{v \in [n]} \Ind(d(u,v) \leq \eta_1-\Delta).\]
Note that $d(u,v) \leq L \|\alpha_u - \alpha_v\|^{\lambda}$ by the Holder smoothness property, so that,
\[\sum_{v \in [n]} \Ind(d(u,v) \leq \eta_1-\Delta) \geq \sum_{v \in [n]} \Ind\left(\|\alpha_v - \alpha_u\|\leq \left(\frac{\eta_1-\Delta}{L}\right)^{1/\lambda}\right).\]

We assumed in our model that $\alpha_v \sim U([0,1]^{d_2})$. 
As a result,
\[\Prob\left(\|\alpha_u-\alpha_v\| \leq \left(\frac{\eta_1-\Delta}{L}\right)^{1/\lambda}\right) \geq \left(\frac{\eta_1-\Delta}{L}\right)^{d_1/\lambda}.\]
As $|\cN_1(u,\eta_1)|$ stochastically dominates a Binomial$(n,\left(\frac{\eta_1-\Delta}{L}\right)^{d_1/\lambda} )$ random variable, the final result follows by Chernoff's bound and union bound.
\end{proof}


\begin{proof}[Proof of Lemma \ref{lemma:nearest_neighbor}]
	
Let
\begin{align*}
	\tilde{F}_{ui} &= \frac{1}{|(\cN_1(u) \times \cN_2(i)) \cap \cE''|}\sum_{v \in \cN_1(u)} \sum_{j\in \cN_2(i)} f(\alpha_v,\beta_j) \Ind((v,j) \in \cE'') \\
\end{align*}
	
We decompose the MSE into two components, using the fact that $(a+b)^2 \leq 2a^2 + 2b^2$,
\begin{align}
MSE &= \frac{1}{nm} \sum_{u \in [n], i \in [m]} (\hat{F}_{ui} - f(\alpha_u, \beta_i))^2 \\
&\leq \frac{2}{nm} \sum_{u \in [n], i \in [m]} \left((\hat{F}_{ui} - \tilde{F}_{ui})^2 + (\tilde{F}_{ui} - f(\alpha_u, \beta_i))^2\right) \label{eq:MSE_decomp}
\end{align}
	
Let us denote event 
\[\cE_6(\epsilon) = \cap_{u,i\in [n]\times[m]} \{|(\cN_1(u,\eta_1) \times \cN_2(i,\eta_2)) \cap \cE''| \geq (1 - \epsilon) p |\cN_1(u,\eta_1)| |\cN_2(i,\eta_2)|\}\]

By Chernoff bound and union bound, 
\[\Prob\left(\cE_6(\epsilon)\right) \geq 1- mn \exp\left(-\frac{\epsilon^2 |\cN_1(u,\eta_1)| |\cN_2(i,\eta_2)| p}{2}\right).\]

Recall that
\begin{align*}
	|\hat{F}_{ui} - \tilde{F}_{ui}| = \frac{1}{|(\cN_1(u) \times \cN_2(i)) \cap \cE''|}\sum_{v \in \cN_1(u)} \sum_{j\in \cN_2(i)} \epsilon_vj \Ind((v,j) \in \cE'')
\end{align*}
where $\epsilon_vj$ are independent and distributed as $N(0,\sigma^2)$. This quantity is thus normally distributed with mean zero and variance $\sigma^2 / |(\cN_1(u) \times \cN_2(i)) \cap \cE''|$. Conditioned on $\cE_4(z_1),\cE_4(z_2),\cE_6(\epsilon),$, with respect to randomness in observation noise in $\{X_{ab}\}_{(a,b) \in \cE''}$, by Hoeffding's inequality,
\begin{align*}
\Prob(|\hat{F}_{ui} - \tilde{F}_{ui}| \geq t) &\leq 2 \exp\left(- \frac{t^2 (1-\epsilon)p z_1 z_2}{2 \sigma^2} \right).
\end{align*}
This results in a high probability upper bound on the first term in \eqref{eq:MSE_decomp}.

Next we bound the second term in \eqref{eq:MSE_decomp} according to
\begin{align}
&\frac{1}{nm} \sum_{u \in [n], i \in [m]} (\tilde{F}_{ui} - f(\alpha_u, \beta_i))^2 \nonumber\\
&= \frac{1}{nm} \sum_{u \in [n], i \in [m]} \left(\frac{1}{|(\cN_1(u) \times \cN_2(i)) \cap \cE''|} \sum_{v,j \in (\cN_1(u) \times \cN_2(i)) \cap \cE''} f(\alpha_v,\beta_j) - f(\alpha_u, \beta_i)\right)^2 \\
&\leq \frac{1}{nm} \sum_{u \in [n], i \in [m]} \frac{1}{|(\cN_1(u) \times \cN_2(i)) \cap \cE''|} \sum_{v,j \in (\cN_1(u) \times \cN_2(i)) \cap \cE''} (f(\alpha_v,\beta_j) - f(\alpha_u, \beta_i))^2 \label{lem:NN_eq1}\\
&\leq \frac{1}{nm} \sum_{u \in [n], i \in [m]} \sum_{v \in \cN_1(u), j \in \cN_2(i)} \frac{\Ind(v,j \in \cE'')}{|(\cN_1(u) \times \cN_2(i)) \cap \cE''|}  2((f(\alpha_v,\beta_j) - f(\alpha_v, \beta_i))^2 + (f(\alpha_v,\beta_i) - f(\alpha_u, \beta_i))^2) \label{lem:NN_eq2}\\
&\leq 2 L^2 \eta_2^{2\lambda} + \frac{1}{nm} \sum_{u \in [n]} \sum_{v \in \cN_1(u)} \sum_{i \in [m]} \frac{2(f(\alpha_v,\beta_i) - f(\alpha_u, \beta_i))^2}{|(\cN_1(u) \times \cN_2(i)) \cap \cE''|}  \sum_{j \in \cN_2(i)} \Ind(v,j \in \cE'') \label{lem:NN_eq3}
\end{align}
where \eqref{lem:NN_eq1} follows from Jensen's inequality, \eqref{lem:NN_eq2} follows from  $(a + b)^2 \leq 2a^2 + 2b^2$, and \eqref{lem:NN_eq3} follows from Holder-smoothness of $f$ and the fact that $\|\beta_i - \beta_j\| \leq \eta_2$ for $j \in \cN_2(i)$.

For all $v,i$, $\sum_{j\in \cN_2(i)} \Ind((v,j) \in \cE'')$ is distributed as a Binomial($|\cN_2(i)|,p$) random variable. By Chernoff's bound and using the assumption that $|\cN_2(i)| \geq z_2$,
\begin{align*}
\Prob\left(\forall u,i \left\{\sum_{j\in \cN_2(i)} \Ind((v,j) \in \cE'') \leq (1 + \delta) p |\cN_2(i)|\right\}\right) \geq 1 - mn\exp\left(-\frac{\delta^2 p z_2}{3}\right)
\end{align*}

In the event that $\sum_{j\in \cN_2(i)} \Ind((v,j) \in \cE'') \leq (1 + \delta) p |\cN_2(i)|$ and 
$|(\cN_1(u,\eta_1) \times \cN_2(i,\eta_2)) \cap \cE''| \geq  |\cN_1(u,\eta_1)| |\cN_2(i,\eta_2)| p (1-\epsilon)$,
by substituting these bounds into \eqref{lem:NN_eq3} and using the definition for $d^2(u,v)$, it follows that 
\begin{align*}
\frac{1}{nm} \sum_{u \in [n], i \in [m]} (\tilde{F}_{ui} - f(\alpha_u, \beta_i))^2 
&\leq 2 L^2 \eta_2^{2\lambda} + \frac{2}{n} \sum_{u \in [n]}  \frac{1 + \delta}{(1 - \epsilon)|\cN_1(u)|} \sum_{v \in \cN_1(u)} d^2(u,v)  \\
&\leq 2 L^2 \eta_2^{2\lambda} + \frac{2(1 + \delta)(\eta_1 + \Delta)^2}{1 - \epsilon}
\end{align*}
where the last inequality follows from the thresholding of $\hat{d}(u,v)$ by $\eta_1$ in the construction of $\cN_1(u)$ along with the good event $\cE_3(\Delta)$, under which $d(u,v) \leq \hat{d}(u,v) + \Delta$.

Putting it all together,
\begin{align*}
MSE &\leq 2 t + 4 L^2 \eta_2^{2\lambda} + \frac{4 (1 + \delta)(\eta_1 + \Delta)^2}{1 - \epsilon}
\end{align*}
with probability at least $1 - mn \exp\left(-\frac{\epsilon^2 p z_1 z_2}{2}\right) - 2 mn \exp\left(- \frac{t^2 (1-\epsilon) p z_1 z_2}{2 \sigma^2} \right) - mn\exp\left(-\frac{\delta^2 p z_2}{3}\right)$.
\end{proof}

\end{document}